\documentclass[letterpaper, 10 pt, conference]{ieeeconf}  

\IEEEoverridecommandlockouts                              

\usepackage{graphicx}
\usepackage{subfig}
\usepackage{amssymb,amsmath,physics}
\usepackage{array,booktabs,arydshln}
\usepackage[cal=boondox,scr=boondoxo]{mathalfa}
\usepackage{algorithm}
\usepackage{algorithmic}
\usepackage[utf8]{inputenc}
\usepackage[english]{babel}
\usepackage{xcolor}
\definecolor{green}{RGB}{3,112,15}
\definecolor{yellow}{RGB}{255,140,0}
\usepackage{hyperref}
\usepackage{cite}
\usepackage{bm}

\newtheorem{definition}{\textbf{Definition}}[section]
\newtheorem{theorem}{\textbf{Theorem}}[section]

\newtheorem{proposition}[theorem]{\textbf{Proposition}}

\title{\Large \bf
Agile Robot Navigation through Hallucinated Learning and Sober Deployment
}

\author{Xuesu Xiao$^{1}$, Bo Liu$^{1}$, and Peter Stone$^{1, 2}$
\vspace{-50pt}
\thanks{$^{1}$Xuesu Xiao, Bo Liu, and Peter Stone are with Department of Computer Science, University of Texas at Austin, Austin, TX 78712 {\tt\scriptsize \{xiao, bliu, pstone\}@cs.utexas.edu}. 
$^{2}$Peter Stone is with Sony AI. 
This work has taken place in the Learning Agents Research
Group (LARG) at UT Austin.  LARG research is supported in part by NSF
(CPS-1739964, IIS-1724157, NRI-1925082), ONR (N00014-18-2243), FLI
(RFP2-000), ARO (W911NF-19-2-0333), DARPA, Lockheed Martin, GM, and
Bosch.  Peter Stone serves as the Executive Director of Sony AI
America and receives financial compensation for this work.  The terms
of this arrangement have been reviewed and approved by the University
of Texas at Austin in accordance with its policy on objectivity in
research.
}
}

\begin{document}

\maketitle
\thispagestyle{empty}
\pagestyle{empty}

\begin{abstract}
Learning from Hallucination (LfH) is a recent machine learning paradigm for autonomous navigation, which uses training data collected in completely safe environments and adds numerous imaginary obstacles to make the environment densely constrained, to learn navigation planners that produce feasible navigation even in highly constrained (more dangerous) spaces. 
However, LfH requires hallucinating the robot perception during deployment to match with the hallucinated training data, which creates a need for sometimes-infeasible prior knowledge and tends to generate very conservative planning. 
In this work, we propose a new LfH paradigm that does not require runtime hallucination---a feature we call ``sober deployment"---and can therefore adapt to more realistic navigation scenarios. This novel Hallucinated Learning and Sober Deployment (HLSD) paradigm is tested in a benchmark testbed of 300 simulated navigation environments with a wide range of difficulty levels, and in the real-world. In most cases, HLSD outperforms both the original LfH method and a classical navigation planner. 
\end{abstract}

\section{INTRODUCTION}
\label{sec::intro}

Machine learning techniques have been recently applied to mobile robot navigation to develop robots that are capable of moving from one point to another within obstacle-occupied environments in a collision-free manner~\cite{pfeiffer2017perception, chen2017socially, faust2018prm, siva2019robot, xiao2020appld, xiao2021toward, liu2021lifelong}. Besides classical planning methods~\cite{quinlan1993elastic, fox1997dynamic}, machine learning approaches can produce effective planners from data instead of hand-crafted rules and heuristics. 

Among the thrust of learning to navigate, Imitation Learning (IL)~\cite{pfeiffer2017perception} and Reinforcement Learning (RL)~\cite{faust2018prm} are the two main streams. While the former requires expert demonstration, the latter learns from trial-and-error. Their initial successes indicate a promising potential future for these data-driven approaches, which do not require sophisticated engineering and in-situ adjustment~\cite{xiao2017uav, xiao2020appld}. However, most learning approaches require a large amount of training data in order to produce good navigation behaviors, especially in challenging 
unseen environments. 

Learning from Hallucination (LfH)~\cite{xiao2021toward} is a recently proposed paradigm to address the difficulty of obtaining high-quality training data. Using LfH, the robot collects training data in an obstacle-free, and thus completely safe, environment with a random exploration policy. During training, the most constrained surrounding obstacle configuration is synthetically projected onto the robot's perception, which allows the effective action taken by the robot in the open space to be the only feasible, and therefore optimal, action. A control policy is learned with training data as if the robot had been moving in those constrained spaces. Thanks to the inherent safety of navigating in a completely open training environment, the robot can autonomously generate a large amount of training data with no human supervision or any costly failure during trial-and-error learning. 

However, one major drawback of LfH \cite{xiao2021toward} is that the perception also needs to be hallucinated during deployment, with the help of a fine-resolution reference path (prior knowledge that is sometimes infeasible to obtain), and it requires other modules to address out-of-distribution scenarios. This runtime hallucination adds extra computation to the perception and becomes ineffective when only sparse future waypoints are available. Furthermore, hallucinating to be always in the most constrained environment during deployment causes the planner to become unnecessarily conservative and fail to adapt to some realistic situations, e.g., an actual open space. 

The Hallucinated Learning and Sober Deployment (HLSD) approach proposed in this work eliminates the necessity of hallucination during deployment and allows the robot to perceive its actual surroundings. This ``sober deployment" relaxes the requirement for a high-resolution reference path and enables the robot to adapt to the real deployment environment. Through a novel hallucination strategy during training, the robot is able to learn from many obstacle configurations as augmentations sampled in addition to the \emph{minimal unreachable set}, which is the \emph{smallest} set of obstacles required to cause the actions performed in the obstacle-free training environment to be optimal, given a specific goal. Our simulated experiment on 300 benchmark testbeds \cite{perille2020benchmarking} and our real-world experiment using a physical robot show that after seeing an extensive amount of carefully designed hallucinated training data, the robot is able to efficiently produce agile maneuvers without runtime hallucination. Superior navigation performance is achieved compared to the original LfH approach with extra access to a high-resolution global path and runtime hallucination, and also to a classical navigation planner. 

\section{RELATED WORK}
\label{sec::related}
This section presents related work in mobile robot navigation, using classical motion planning  and recent machine learning techniques. 

\subsection{Classical Motion Planning}
In terms of classical motion planning, mobile robot navigation is the problem of moving a robot from one point to another within an obstacle-occupied space in a collision-free manner. Such planning usually happens in the robot Configuration Space (C-Space) \cite{latombe2012robot} and produces (asymptotically) optimal motion plans based on a predefined metric, such as maximum clearance, shortest path, or a combination thereof~\cite{lavalle2006planning}. These metrics are optimized in terms of C-space representation (e.g. cellular decomposition~\cite{lavalle2006planning}), global planning (e.g. Dijkstra's), and local planning. For example, the optimization-based Elastic Bands (E-Band) planner~\cite{quinlan1993elastic} uses repulsive forces generated from obstacles to deform and optimize an initial trajectory, primarily to achieve maximum clearance. The Dynamic Window Approach (DWA)~\cite{fox1997dynamic} samples feasible actions and scores them based on a weighted score of distance to obstacles, closeness to a global path, and progress toward the goal. The new hallucination strategy proposed in this paper assumes the planner learned from hallucinated training data primarily seeks the shortest path. We leave hallucination strategies that optimize for maximum clearance for future work.

\subsection{Machine Learning for Navigation}

Despite decades of effort to develop classical autonomous navigation systems \cite{quinlan1993elastic, fox1997dynamic}, machine learning has recently shown promise for creating competitive end-to-end planners with obstacle avoidance \cite{pfeiffer2017perception}, enabling terrain-based navigation \cite{siva2019robot, xiao2021learning}, allowing robots to move around humans \cite{chen2017socially}, and tuning parameters for classical navigation systems \cite{xiao2020appld, wang2021appli, xu2021applr}. All these learning methods require either extensive or high-quality training data, such as that derived from trial-and-error exploration or from human demonstrations~\cite{xiao2020motion}. 

To address the difficulty in obtaining extensive or high-quality training data, self-supervised LfH \cite{xiao2021toward} collects training data in an obstacle-free environment with complete safety, and then learns a local navigation policy through synthetically projecting the most constrained C-space ($C_{obst}^*$) onto the robot perception. The most constrained C-space corresponds to the obstacle configuration with as many obstacles as possible such that the executed motion is the \emph{only} feasible and therefore \emph{optimal} motion plan: if any more obstacles were to be added, then the motion plan would no longer be feasible. 
However, LfH also requires hallucination of $C_{obst}^*$ during deployment. This hallucination ``on-the-fly'' requirement assumes prior knowledge such as a relatively high-resolution global path is available. Furthermore, this previous work \cite{xiao2021toward} only works well when the linear velocity in the training set is mostly constant, due to the fact that the robot always hallucinates navigating in the most constrained scenarios and the trained local planner is therefore relatively conservative. Varying speeds in the same most constrained space will lead to ambiguity, as shown in our experiments (Section \ref{sec::experiments}). In this work, we eliminate the necessity of hallucination ``on-the-fly'', and therefore the requirement of an available high-resolution global path. During sober deployment, the local planner only needs a single local goal point, along with the real, rather than hallucinated, perception. 
In addition, the novel hallucination technique also teaches the robot to vary its speed in response to the real environment, instead of the hallucinated most constrained spaces.

\vspace{-5pt}
\section{APPROACH}
\label{sec::approach}

In this section, we present our Hallucinated Learning and Sober Deployment (HLSD) technique for mobile robot navigation. In Sec. \ref{sec::point_robot}, we start with a simplified example, which assumes the robot to be a point mass that follows a relatively simple path, to introduce the idea of a minimal unreachable set, $C_{obst}^{min}$, for a given optimal plan. We provide necessary conditions for being a minimal unreachable set, as the basis for many unreachable sets for the optimal plan. 
In Sec. \ref{sec::realistic_robot}, we show how we use one special case, $C_{obst}^{\overline{min}}$, to represent all other minimal unreachable sets. We also present how we adapt the point mass example to deal with real-world robots. In Sec. \ref{sec::sampling}, we introduce a sampling method to generate augmentations to the representative minimal unreachable set ($C_{obst}^{\overline{min}}$) to generalize for sober deployment. 

\subsection{Point Robot Example}
\label{sec::point_robot}
We first present a simplified example, which assumes a point mass robot going through three non-colinear configurations. In this case, the robot’s workspace is exactly the C-space. We use the same notation used by Xiao {\em et al}. to formalize LfH~\cite{xiao2021toward}: given a robot's C-space partitioned by unreachable (obstacle) and reachable (free) configurations, $C = C_{obst} \cup C_{free}$, 
the classical motion planning problem is to find a function $f(\cdot)$ that can be used to produce optimal plans $p=f(C_{obst}~|~c_c, c_g)$ that result in the robot moving from the robot's current configuration $c_c$ to a specified goal configuration $c_g$ without intersecting (the interior of) $C_{obst}$. Here, a plan $p \in \mathcal{P}$ is a sequence of low-level actions $\{u_i\}_{i=1}^{t}$ ($u_i \in \mathcal{U}$, $\mathcal{P}$ and $\mathcal{U}$ are the robot’s plan and action space, respectively). Considering the ``dual" problem of finding $f(\cdot)$, 
LfH~\cite{xiao2021toward} includes a method to find the (unique) \emph{most constrained} unreachable set corresponding to $p$, $C_{obst}^*$. In this work, however, we introduce the definition of a (not unique) \emph{minimal} unreachable set, $C_{obst}^{min}$, corresponding to $p$:  

\begin{definition}
$\mathcal{C}_{obst}^{min} \doteq \{C_{obst}^{min} \mid \forall c \in C_{obst}^{min}, 
f(C_{obst}^{min}\setminus\{c\}~|~c_c, c_g) \neq f(C_{obst}^{min}~|~c_c,c_g)\}$
\label{def::min}
\end{definition}
In other words, every member, $C_{obst}^{min}\in \mathcal{C}_{obst}^{min}$, is a minimal set of obstacles that lead to $p$ being an optimal plan.\footnote{$C_{obst}^{min}$ is minimal in the sense that no subset of it also leads to the same optimal plan, i.e. nothing can be removed such that $p$ remains optimal. } 
Any path that arises from an optimal plan can be approximated by connected line segments. The simplest case of a robot path following an optimal plan $p$ to move from $c_c$ to $c_g$ (except a straight line) is composed of configurations on two straight line segments, $c_c-c_m$ and $c_m-c_g$.
The intermediate turning point is defined as $c_m$. Since $c_m$ is one configuration on the robot's path, we say $p$ \emph{goes through} $c_m$ (Fig. \ref{fig::c_obst^min}). 
In the following, given a point-mass robot moving from $c_c$ to $c_g$ according to some optimal plan $p$ computed by $f(\cdot)$, 
we show two necessary conditions ($\Longrightarrow$) for an unreachable set ($C_{obst}$) to be a minimal unreachable set ($C_{obst}^{min}$),  to aid in identifying the representative unreachable set ($C_{obst}^{\overline{min}}$) used in Sec. \ref{sec::realistic_robot}. 


\begin{proposition}
If $c \in C_{obst}^{min}$, then the optimal plan for the unreachable set $C_{obst}^{min}\setminus \{c\}$, $\hat{p} = f(C_{obst}^{min}\setminus \{c\}~|~c_c, c_g)$, must go through $c$.
\label{prop::1}
\end{proposition}
\begin{proof}
Assume otherwise. Since $\hat{p} = f(C_{obst}^{min}\setminus\{c\}~|~c_c, c_g) \neq f(C_{obst}^{min}~|~c_c,c_g) = p$ (\textit{\textbf{Definition} \ref{def::min}}), for $C_{obst}^{min}\setminus\{c\}$, 
the path arising from $\hat{p}$ is shorter than the one arising from $p$. Since we assume $\hat{p}$ does not go through $c$, adding $c$ back to $C_{obst}^{min}\setminus\{c\}$ does not affect the feasibility of $\hat{p}$ for $C_{obst}^{min}$ and does not change the path length. The path arising from $\hat{p}$ is still shorter than the one arises from $p$ in $C_{obst}^{min}$, thus contradicting the optimality of $p$ for $C_{obst}^{min}$.
\end{proof}


\begin{proposition}
($\Longrightarrow$ 1) 
$\forall C_{obst} \in \mathcal{C}_{obst}^{min}, c_m\in C_{obst}$.\footnote{$C$ is a topological space, $C_{obst}$ is a closed set, and $C_{free}=C\setminus C_{obst}$ is an open set. $c_m$ is a boundary point of both $C_{obst}$ and $C_{free}$. The robot can come arbitrarily close to the obstacles while remaining in $C_{free}$~\cite{lavalle2006planning}.}
\label{prop::c_m}
\end{proposition}
\begin{proof}
Consider any circle $\mathcal{B}(c_m, \epsilon)$ that centers at $c_m$ with radius $\epsilon$. $\mathcal{B}$ intersects $c_c-c_m$ at $c_1$ and $c_m-c_g$ at $c_2$ (Fig. \ref{fig::c_obst^min} a). Assume there exists no configuration $c \in \mathcal{B}$ such that $\exists C_{obst} \in \mathcal{C}_{obst}^{min}$ and $c \in C_{obst}$. Consider the path $c_c-c_1-c_2-c_g$: since $c_c-c_m-c_g$ is feasible, then $c_c-c_1$ and $c_2-c_g$ are feasible. Moreover, by assumption $c_1-c_2$ is also feasible. But $c_c-c_1-c_2-c_g$ is shorter than $c_c-c_m-c_g$ since $c_1-c_2$ is shorter than $c_1-c_m-c_2$ due to triangle inequality. This contradicts the optimality of $c_c-c_m-c_g$. 
Therefore, $\exists c\in \mathcal{B}$ that belongs to some $C_{obst} \in \mathcal{C}_{obst}^{min}$. Since this is true for $\lim_{\epsilon \rightarrow 0}\mathcal{B}(c_m, \epsilon)$, and $C_{obst}$ is a closed set~\cite{lavalle2006planning}, the limit point $\mathcal{B}(c_m, 0) = c_m \in C_{obst}$. 
\end{proof}

We name the union of all robot configurations in the triangle defined by $c_c$, $c_m$, and $c_g$ as $G_1$. On the other side of line segment $c_c-c_g$, we name the union of all configurations in the half ellipse, whose focal points locate at $c_c$ and $c_g$, and whose major axis length is $|c_c c_m|+|c_mc_g|$, as $G_2$. We define $G = G_1 \cup G_2$ (the grey area in Fig. \ref{fig::c_obst^min}). 


\begin{proposition}
($\Longrightarrow$ 2) 
$\forall C_{obst} \in \mathcal{C}_{obst}^{min}$, $\forall c \in G$, $\{c_p~|~c_p \text{ on line segments } c_c-c-c_g\}\cap C_{obst} \neq \emptyset$ 
\label{prop::iff}
\end{proposition}
\begin{proof}
Assume $\exists c \in G$, $\{c_p~|~c_p \text{ on line segments } c_c-c-c_g\}\cap C_{obst} = \emptyset$. Then $c_c-c-c_g$ is a feasible path. The length of $c_c-c-c_g$ is the sum of distances from $c$ (inside ellipse) to the two ellipse focal points $c_c$ and $c_g$, which is, per definition of an ellipse, shorter than its major axis length $|c_c c_m|+|c_mc_g|$. This contradicts the optimality of $p$. 
\end{proof}

Based on the two necessary conditions of $C_{obst} \in \mathcal{C}_{obst}^{min}$ (\textit{\textbf{Proposition} \ref{prop::c_m}} and \textit{\ref{prop::iff}}), one class of minimal unreachable sets could be simply constructed by connecting $c_m$ with some point $c_e$ on the left boundary of the ellipse with a straight line $c_m-c_e$ ($C_{obst}^{min_1}$ in Fig. \ref{fig::c_obst^min} a), or two line segments $c_m-c_g(c_c)$ and $c_g(c_c)-c_e$ ($C_{obst}^{min_2}$ in Fig. \ref{fig::c_obst^min} b), 
if not all configurations on $c_m-c_e$ are in $G$. In particular, for efficiency, we simply represent all minimal unreachable sets with a special case, $C_{obst}^{\overline{min}}$ (Fig. \ref{fig::c_obst^min} d), which is all configurations along the straight line $c_m$ and $c_m'$ (the reflective symmetry point of $c_m$ with respect to $c_c-c_g$). Empirical evidence of this approximation's sufficiency for the purpose of learning will be provided in Sec. \ref{sec::experiments}. Here, we further provide one more observation to help develop intuition 
regarding how to identify a $C_{obst}^{min}$.

\begin{proposition}
$\forall C_{obst}^{min}$, $\forall c \in C_{obst}^{min}$, $c \in G$. 
\label{prop::outside}
\end{proposition}

\begin{proof}
Assume $c \in C_{obst}^{min}$ and $c \notin G$, only two possibilities exist: 


(1) $c$ is outside the entire (left and right half) ellipse whose focal points locate at $c_c$ and $c_g$, and whose major axis length is $|c_c c_m|+|c_mc_g|$: 
based on \textit{\textbf{Proposition} \ref{prop::1}}, the optimal plan $\hat{p} = f(C_{obst}^{min}\setminus \{c\}~|~c_c, c_g)$ must go through $c$. The shortest possible path, which goes through $c$, is $c_c-c-c_g$, if it exists. However, based on the definition of ellipse, for any $c$ outside the ellipse, $|c_cc|+|cc_g|>|c_cc_m|+|c_mc_g|$. This contradicts the optimality of $\hat{p}$. 

(2) $c$ is inside the entire ellipse, but outside $G$: $c$ must be in the right half of the ellipse but outside $G_1$ (Fig. \ref{fig::c_obst^min} e). Again, $\hat{p}$ must go through $c$. The shortest possible path which goes through $c$ is $c_c-c-c_g$, if it exists, which must intersect either $c_c-c_m$ or $c_m-c_g$ at some point $c_i$. Due to substructure optimality (i.e. a sub-path of a shortest path is still a shortest path), the shortest path between any points on $c_c-c_m-c_g$ must be its sub-path. 
If $c_i$ is on $c_c-c_m$, then the shortest path from $c_i$ to $c_g$ must be $c_i-c_m-c_g$. $c_c-c-c_i-c_m-c_g$ is longer than $c_c-c_m-c_g$, and therefore not optimal. If $c_i$ is on $c_m-c_g$, the shortest path between $c_c$ and $c_i$ must be $c_c-c_m-c_i$. $c_c-c_m-c_i-c-c_g$ is longer than $c_c-c_m-c_g$, and therefore not optimal. In both cases, the shortest possible path going through $c$ is longer than $c_c-c_m-c_g$. This contradicts the optimality of $\hat{p}$. 

Therefore, $\forall C_{obst}^{min}$, $\forall c \in C_{obst}^{min}$, $c \in G$. 
\end{proof}

\begin{figure}
  \centering
  \includegraphics[width=\columnwidth]{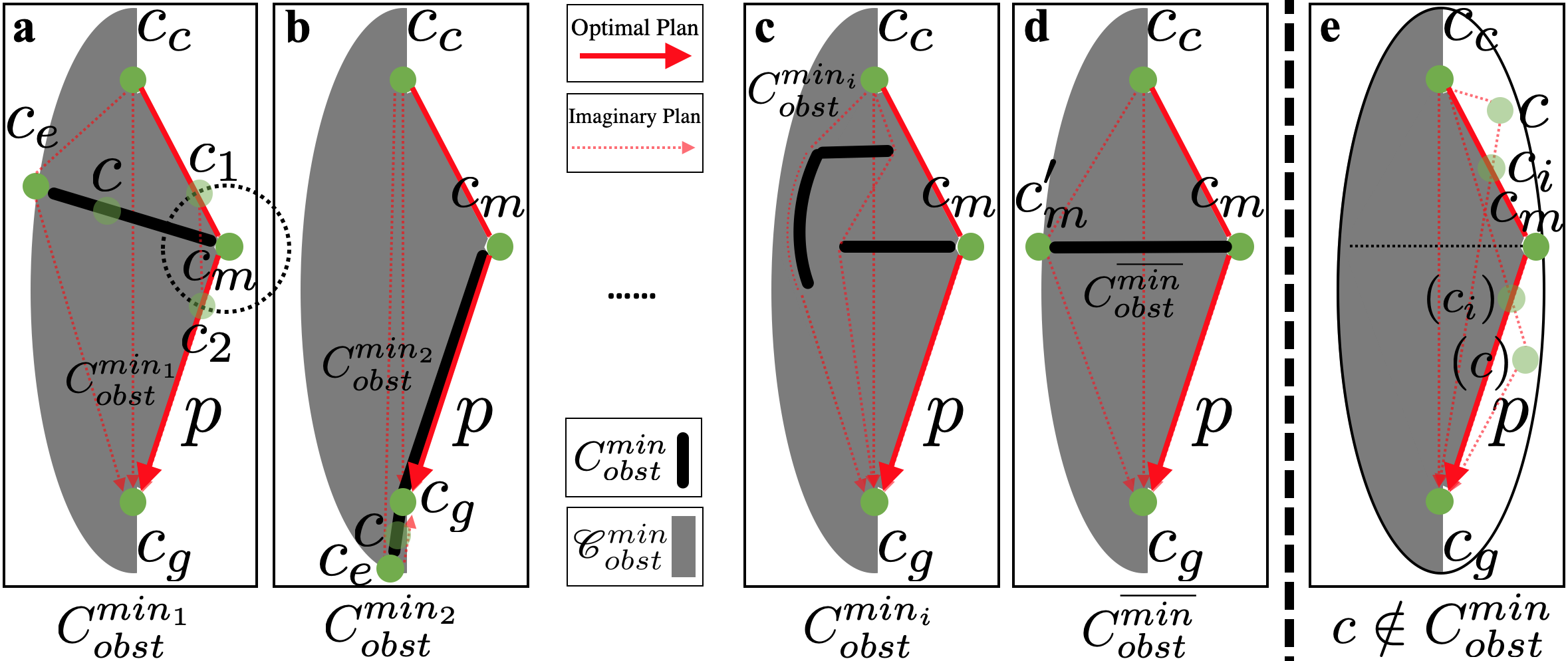}
  \caption{
  $C_{obst}^{min_1},~C_{obst}^{min_2},~...~,~C_{obst}^{min_i},~C_{obst}^{\overline{min}} \in \mathcal{C}_{obst}^{min}$. 
}
  \label{fig::c_obst^min}
  \vspace{-10pt}
\end{figure}

\subsection{Realistic Nonholonomic Robot}
\label{sec::realistic_robot}
Based on the propositions discussed in Sec. \ref{sec::point_robot}, we present the hallucination technique for a realistic robot. In realistic scenarios, we approximate all $C_{obst}^{min}$ in $\mathcal{C}_{obst}^{min}$ with $C_{obst}^{\overline{min}}$. 
We hypothesize that all $C_{obst}^{min}$ are sufficiently similar that using just one leads to learning that is just as good as if we used them all, especially when (1) $c_c$, $c_m$, and $c_g$ extracted from realistic trajectories are close to each other, and (2) $C_{obst}$ is instantiated in terms of discrete LiDAR beams. Empirical evidence of this sufficiency will be provided in Sec. \ref{sec::experiments}. 
However, a realistic robot cannot be modeled as a simple point mass because: (1) the size is not negligible and we need a path for the \emph{center} of the robot that causes no part of the robot to collide with an obstacle; (2) nonholonomic robots cannot change motion direction instantly, so we need to generalize to a \emph{continuously} turning path from the \emph{piece-wise} point mass example. 
In Sec. \ref{sec::realistic_robot}, we aim to adapt the point mass example (Sec. \ref{sec::point_robot}) to real-world robots, and therefore need to address these differences. 

To address (1), we define one point to the left and another to the right of the centroid of the robot, offset by the robot width, as footprint points, as shown in Fig. \ref{fig::real_robot}. The instantaneous linear velocity along the line between these two footprint points is zero for nonholonomic vehicles. The polygon defined by a sequence of footprint points must belong to free space. The actual path executed by the optimal plan $p$ follows the middle of this area. To address (2), we define $c_{\omega}$ as configurations between the current and goal configurations with a non-zero angular velocity ($\omega \neq0$). Given the current configuration, each $c_\omega$, and each $c_\omega$'s next configuration, their left or right footprint points are treated as $c_c$, $c_m$, and $c_g$ in the \emph{point robot} case: based on the sign of $\omega$, the robot turns left or right, and the left or right footprint points are chosen. 
For each triple of \emph{point robot} $c_c$, $c_m$, and $c_g$,   
for efficiency, we approximate all different $C_{obst}^{min}$ with $C_{obst}^{\overline{min}}$ (Fig. \ref{fig::c_obst^min} d). For a realistic optimal plan $p$ with \emph{actual} $c_c$ (current) and $c_g$ (goal) and multiple \emph{point robot} $c_c-c_m-c_g$ triples in between, we define the union of all $C_{obst}^{\overline{min}}$ as $\mathcal{C}_{obst}^{\overline{min}}$. 

In particular, we define an ``opposite'' function of $f(\cdot)$: 
$\mathcal{C}_{obst}^{\overline{min}} = o(p~|~c_c, c_g)$
as the \emph{minimal} hallucination function.\footnote{Note the inverse function $f^{-1}(\cdot)$ does not exist. Technically, $c_g$ can be uniquely determined by $p$ and $c_c$, but we include it as an input to $o(\cdot)$ for notational symmetry with $f(\cdot)$.} This function finds $\mathcal{C}_{obst}^{\overline{min}}$ based on the fact that  $p$ is an optimal plan. 
As visualized in Fig. \ref{fig::real_robot}, $C_{obst}$ is instantiated in terms of discrete LiDAR range readings. 
For each LiDAR beam, we define a minimum and a maximum range, which limit the possible range readings of this particular beam based on the optimal plan $p$. The minimum range of all beams is determined by the left and right boundary, as configurations within the boundary must be in $C_{free}$. We project all $C_{obst}^{\overline{min}}$ onto the corresponding LiDAR beams and for the beams directly intersect any $C_{obst}^{\overline{min}}$, the maximum range is set as the distance from the robot to the intersection point. For other beams, the maximum range is simply the LiDAR's physical limit, or manually pre-processed to a certain threshold. 

\begin{figure}
  \centering
  \includegraphics[width=1\columnwidth]{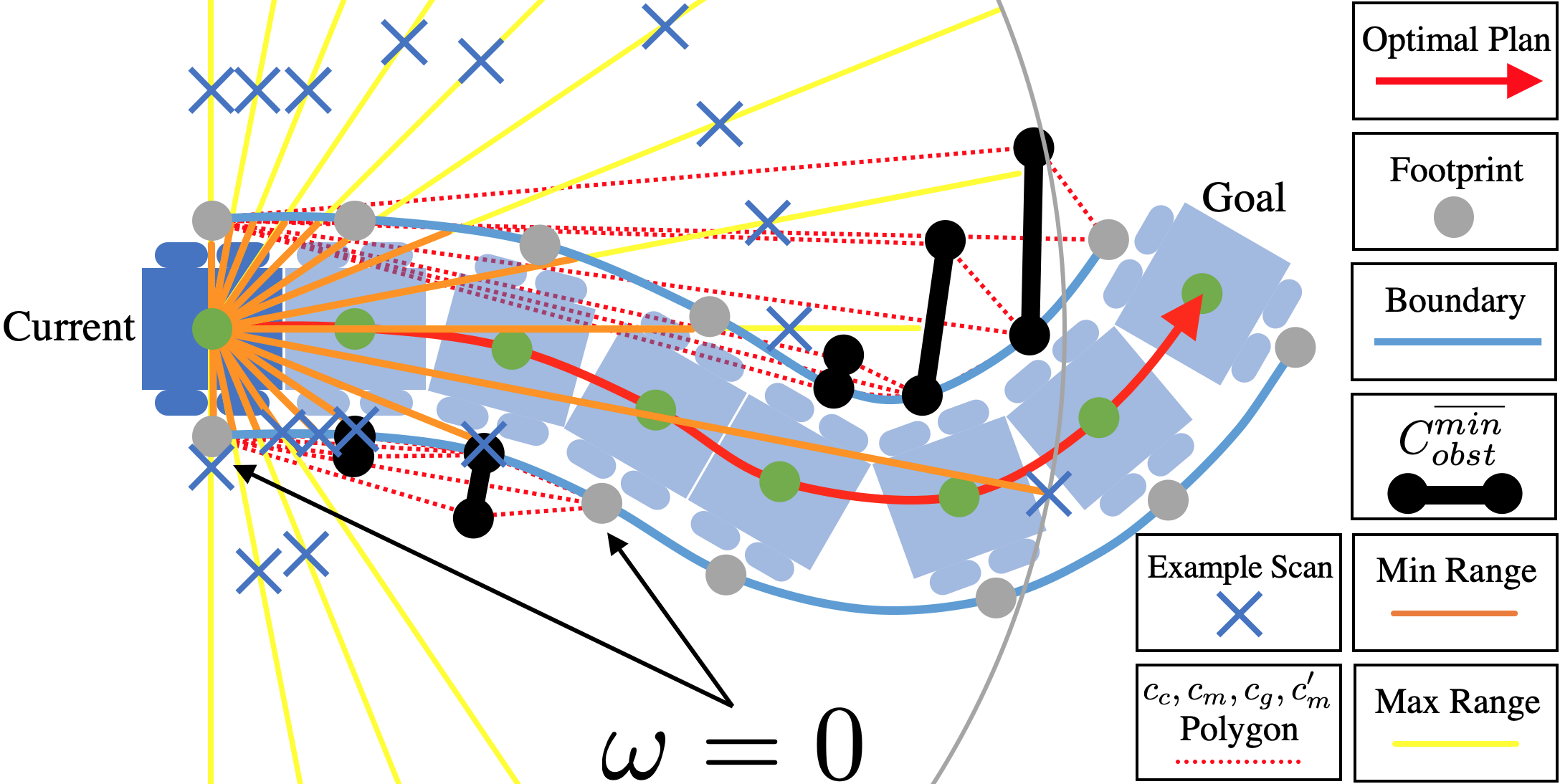}
  \caption{Applying the point robot example to realistic robot case: $C_{obst}^{min}$ is instantiated as LiDAR beams whose maximum range is determined by ray casting.}
  \label{fig::real_robot}
  \vspace{-10pt}
\end{figure}

\subsection{Sampling between Min and Max Range}
\label{sec::sampling}
Given the representative minimal unreachable set, we want to find all possible unreachable sets, or their sensor readings, that could lead to $p$ being the optimal plan.
Based on the LiDAR scan with a minimum and maximum range for each beam (end of Sec. \ref{sec::realistic_robot}), a sampling strategy is devised to create many obstacle sets, $C_{obst}$, in which $p$ is optimal. 

Our sampling strategy aims at creating different range readings that (1) resemble real-world obstacles, and (2) respect uncertainty/safety. For (1), most real-world obstacles have a certain footprint, and their surface contribute to continuity among neighboring beams. Starting from the first beam, a random range is sampled between min and max with a uniform distribution. Moving on to the neighboring beam, with a probability $\alpha$, we increase, or decrease, the previous range by a small random amount, and assign the value to the current beam. This practice is to simulate the continuity in neighboring beams. We make sure this value is within the min and max range of that beam. With probability $1-2\alpha$, we start from scratch and randomly sample between this beam's min and max values. This practice is to simulate the scenarios where the next beam misses the current obstacle, and reaches another one or does not reach any obstacle at all. For (2), we add an offset value as a function of the optimal plan $p$ to the ranges. For example, given a faster speed of $p$, a larger positive offset is added to the range reading (obstacles are farther away), because faster motion is correlated with more uncertainty or less safety (details can be found in Sec. \ref{sec::implementation}). One example scan sample is shown in Fig. \ref{fig::real_robot} as blue $\times$'s. 

The entire HLSD pipeline is described in Alg. \ref{alg::sober}. The inputs to the algorithm are a random exploration policy $\pi_{rand}$ in open space; the minimal hallucination function $o(\cdot)$; a \emph{sampling count} of hallucinated $C_{obst}$ to be generated per data point; the \emph{offset}$(\cdot)$ function; the probability $\alpha$; and a parameterized planner $f_{\theta}(\cdot)$. For every data point $(p, c_c, c_g)$ in $\mathcal{D}_{raw}$ collected using $\pi_{rand}$ in open space (line 2), we hallucinate $\mathcal{C}_{obst}^{\overline{min}}$ (line 5) and generate the min and max values for every LiDAR beam (line 6). Lines 8--15 correspond to the sampling technique to generate random laser scans. We instantiate $C_{obst}$ as LiDAR readings $L$ and add it to $\mathcal{D}_{train}$ (line 16). This process is repeated \emph{sampling count} times for every data in $\mathcal{D}_{raw}$. Finally, we train $f_{\theta}(\cdot)$ with supervised learning (line 19). This hallucinated learning enables sober deployment with perception of the real configuration $C_{obst}^{real}$ without runtime hallucination (Lines 21--22). 

\begin{algorithm}[b!]
 \caption{Hallucinated Learning and Sober Deployment}
 \begin{algorithmic}[1]
 \renewcommand{\algorithmicrequire}{\textbf{Input:}}
 \REQUIRE $\pi_{rand}$, $o(\cdot)$, \emph{sampling count}, \emph{offset}($\cdot$), $\alpha$, $f_{\theta}(\cdot)$
\\\hrulefill
  \STATE // \textbf{Hallucinated Learning}
  \STATE collect motion plans $(p, c_c, c_g)$ from $\pi_{rand}$ in free space and form raw data set $\mathcal{D}_{raw}$ 
  \STATE $\mathcal{D}_{train} \leftarrow \emptyset$
  \FOR {every $(p, c_c, c_g)$ in $\mathcal{D}_{raw}$}
  \STATE hallucinate $\mathcal{C}_{obst}^{\overline{min}} = o(p~|~c_c, c_g)$
  \STATE generate LiDAR range $L_{min}$ and $L_{max}$ with $\mathcal{C}_{obst}^{\overline{min}}$
    \FOR {iter = 1 : \emph{sampling count}}
    \STATE $L \leftarrow \emptyset$
    \STATE $l^1 \sim [l_{min}^1, l_{max}^1]$, $l^1 \leftarrow l^{1} + \text{\emph{offset}}(p)$
    \STATE add $l^1$ to $L$, $l^{last}\leftarrow l^1$
    \FOR{i = 2 : $|L|$}
        \STATE increase, decrease $l^{last}$ by a randomly selected small amount, or $l^i \sim [l_{min}^i, l_{max}^i]$, $l^i \leftarrow l^{i} + \text{\emph{offset}}(p)$, 
        with probability $\alpha$, $\alpha$, and $1-2\alpha$, respectively, and assign to $l^i$
        \STATE make sure $l^i \in [l_{min}^i, l_{max}^i]$
        \STATE add $l_i$ to $L$, $l^{last}\leftarrow l^i$
    \ENDFOR
    \STATE $C_{obst} \leftarrow L$, $\mathcal{D}_{train}=\mathcal{D}_{train}\cup(C_{obst}, p, c_c, c_g)$
    \ENDFOR
  \ENDFOR
  \STATE train $f_{\theta}(\cdot)$ with $\mathcal{D}_{train}$ by minimizing the error $\mathbb{E}_{(C_{obst}, p, c_c, c_g) \sim \mathcal{D}_{train}} \big[ \ell(p, f_\theta(C_{obst}~|~c_c, c_g))\big]$
\\\hrulefill
  \STATE // \textbf{Sober Deployment} (each time step)
  \STATE receive $C_{obst}^{real}, c_c, c_g$
  \STATE plan $p = \{u_i\}_{i=1}^{t} = f_{\theta}(C_{obst}^{real}~|~c_c, c_g)$
 \RETURN $p$
 \end{algorithmic}
 \label{alg::sober}
 \end{algorithm}

\section{EXPERIMENTS}
\label{sec::experiments}
Simulated and physical experiments are conducted to validate our hypothesis that HLSD can achieve better performance (faster, smoother, safer) than a classical method and LfH with runtime hallucination. In our experiments, we use a Clearpath Jackal robot, a four-wheeled, differential-drive, nonholonomic,  Unmanned Ground Vehicle (UGV), running the Robot Operating System (ROS) \texttt{move\textunderscore base} navigation stack. Its DWA local planner is replaced with HLSD. The robot navigates without a map. The global planner (Dijkstra's algorithm) assumes unknown regions are free and replans when obstacles are perceived. The local environment is known to the local planner. 

\subsection{Implementation}
\label{sec::implementation}
In order to instantiate Alg. \ref{alg::sober}, with \pmb{$o(\cdot)$} described in detail in Sec. \ref{sec::realistic_robot}, one still needs to define \pmb{$f_\theta(\cdot)$}, \pmb{$\pi_{rand}$}, \textbf{\emph{sampling count}}, \pmb{$\alpha$}, and \textbf{\emph{offset}}\pmb{$(\cdot)$}: 

\pmb{$f_\theta(\cdot)$}: For $p = \{u_i\}_{i=1}^{t} = f_{\theta}(C_{obst}^{real}~|~c_c, c_g)$, instead of requiring a high-resolution global path from the global planner (Dijkstra's) to construct $C_{obst}^*$ in LfH~\cite{xiao2021toward}, we only query a single local goal $c_g$ on the global path 1m away from the robot at each time step, and $c_c$ is the origin in the robot frame (orientation is ignored for simplicity). The UGV is equipped with a 720-dimensional 2D LiDAR with a 270$^\circ$ field of view, and we clip the maximum range to 1m to reduce the input space ($C_{obst}^{real}$). The planning horizon $t$ of $p$ is set to 1, i.e. only a single action $p = \{u_1\} = \{\left(v_1, \omega_1\right)\}$ (linear and angular velocity) is produced. We use the same three-layer neural network, with 256 hidden neurons and ReLU activation for each layer as in LfH~\cite{xiao2021toward}. 

\pmb{$\pi_{rand}$}: $\pi_{rand}$ randomly picks a target $\left(\hat{v}, \hat{\omega}\right)$ pair and commands the robot to reach that speed with constant increments/decrements considering acceleration limit. After reaching $\left(\hat{v}, \hat{\omega}\right)$, $\pi_{rand}$ keeps that command with some probability (0.9) or otherwise generates a new target command. We limit the output $v \in [0, 1.0]$m/s and $\omega \in [-1.57, 1.57]$rad/s. During training in a simulated open space, control inputs ($v$ and $\omega$) and robot configurations ($x$, $y$, and $\psi$) are recorded. Unlike the dataset collected by LfH~\cite{xiao2021toward}, which mostly contains $v=0.4$m/s, our dataset contains a variety of $v$ values. 12585 data points are collected in a 505s real-time simulation, including a variety of motions in an open space. 

\textbf{\emph{sampling count}}, \pmb{$\alpha$}, and \textbf{\emph{offset}}\pmb{$(\cdot)$}: We set \emph{sampling count} to 10 and $\alpha$ to 0.48. The \emph{offset}($\cdot$) function linearly maps current $v$ in the range [0.3, 1.0]m/s to an offset value between [0, 1]m. Considering the fact that a real robot also needs to turn even in open space due to nonholonomic constraints, as opposed to an ideal point mass robot which does not, we also hallucinate $C_{obst}=\emptyset$ for configurations where $v>0.8$m/s. Considering highly constrained spaces, we additionally hallucinate the most constrained $C_{obst}^*$ for configurations where $v<0.3$m/s. 
So $|\mathcal{D}_{train}|=12*|\mathcal{D}_{raw}|$. 
Training takes less than three minutes on a NVIDIA GeForce GTX 1650 laptop GPU. 

After computing $p$ with Alg. \ref{alg::sober}, we use the same Model Predictive Control model as in LfH \cite{xiao2021toward} to check for collisions. When a collision is detected, the robot enters a two-phase recovery behavior: it first queries the global path immediately in front of the robot, and rotates to align the robot heading with the tangential direction of the global path. If this recovery behavior is still not safe as determined by the collision checking, the robot backs up at $v=-0.2$m/s. Since LfH learns only from the \emph{most constrained} C-space, it requires runtime hallucination to match with training data, and a Turn in Place module to drive the robot out of out-of-distribution scenarios. Neither of those components are required by HLSD. 
Because our dataset contains varying $v \in [0, 1.0]$m/s, the LfH speed modulation module to adapt mostly constant $v$ to real environments is not necessary either. The robot is able to react to the real obstacle configuration by using $f_{\theta}(\cdot)$ alone because $\mathcal{D}_{train}$ covers a wide range of distributions. 

\subsection{Simulated Experiments}

\begin{figure*}
  \centering
  \includegraphics[width=2\columnwidth]{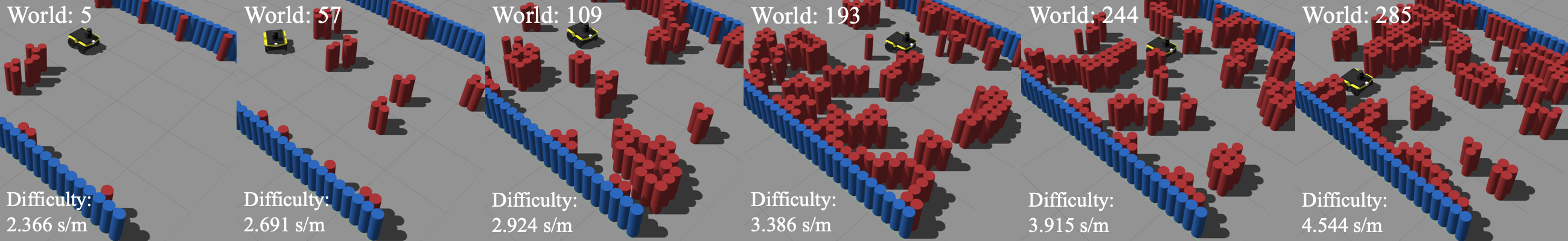}
  \vspace{-4pt}
  \caption{Simulated Environments of Different Difficulties in the BARN Dataset~\cite{perille2020benchmarking}}
  \label{fig::simulated_envs}
  \vspace{-18pt}
\end{figure*}

\begin{figure}
  \centering
  \includegraphics[width=1\columnwidth]{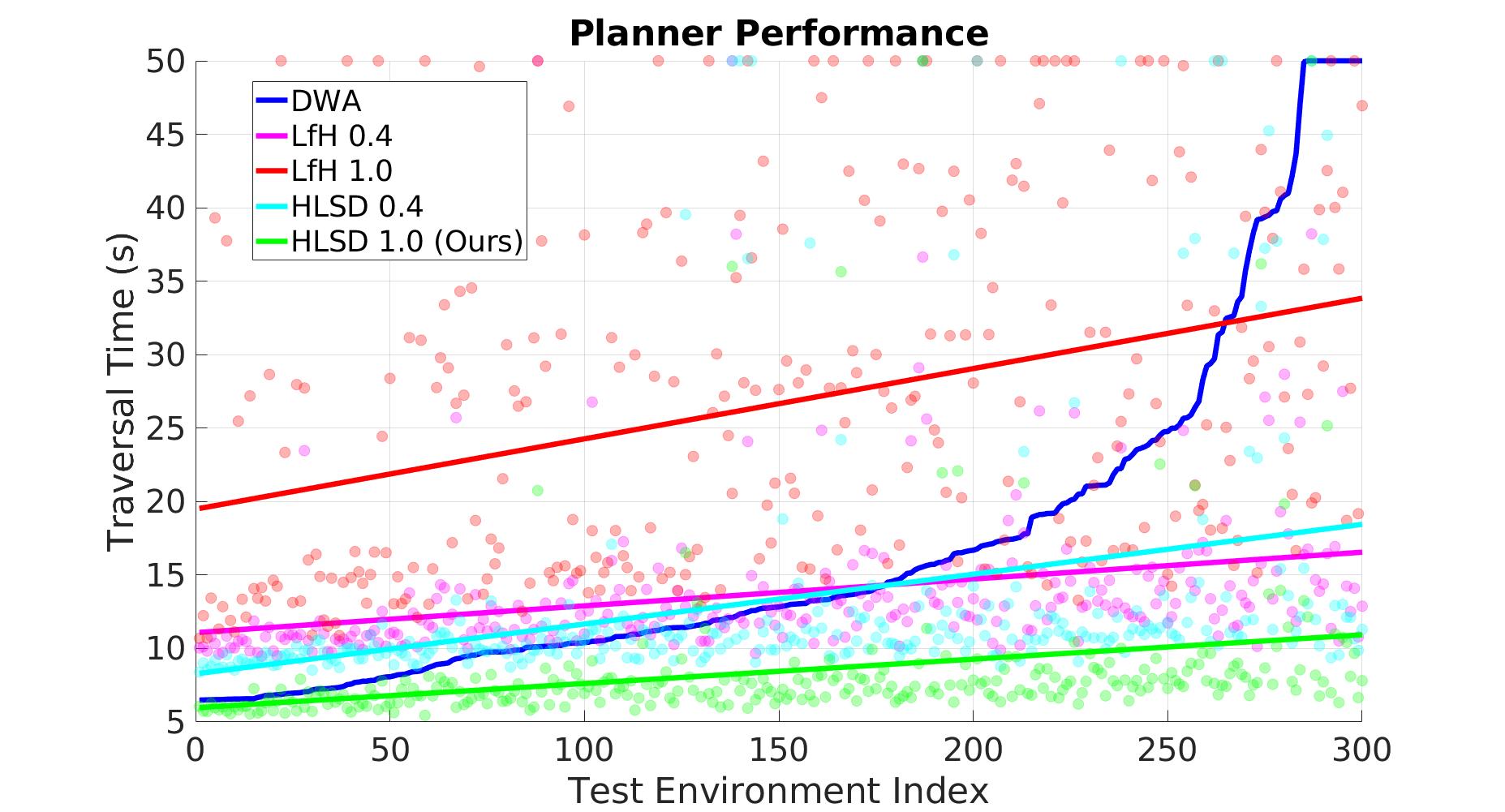}
  \vspace{-18pt}
  \caption{Simulation Results}
  \label{fig::simulation}
  \vspace{-15pt}
\end{figure}

We first use the Benchmark Autonomous Robot Navigation (BARN) dataset with 300 navigation environments and quantified difficulty levels~\cite{perille2020benchmarking} (example environments shown in Fig. \ref{fig::simulated_envs}) to compare (1) DWA, (2) original LfH trained on the mostly constant 0.4m/s dataset with hallucinated $C_{obst}^*$ (LfH 0.4), (3) LfH trained on our varying speed 1.0m/s dataset with $C_{obst}^*$ (LfH 1.0), (4) HLSD trained on the 0.4m/s dataset with augmentations to $C_{obst}^{\overline{min}}$ (HLSD 0.4), and (5) HLSD trained on the 1.0m/s dataset with $C_{obst}^{\overline{min}}$ (HLSD 1.0). 
The baselines are chosen for the following reasons: DWA is a widely used local planner and its implementation is available through the ROS \texttt{move\textunderscore base} navigation stack. We only consider similar self-supervised approaches and exclude methods that require expert supervision~\cite{pfeiffer2017perception, pokle2019deep}, because HLSD does not require any expert motion trajectories. 
Although DWA's default max linear velocity is 0.5m/s, for a fair comparison, we set DWA's max linear velocity to the same as HLSD's (1.0m/s).  We find that by also doubling DWA's default sampling rate for linear and angular velocity (12 and 40), the robot's performance is roughly the same as when using the default parameters (but at double the speed). 
Simulated results are shown in Fig. \ref{fig::simulation}.

In each of the 300 navigation environments generated by Cellular Automaton, the robot navigates between a specified start and goal location without a map. We record the traversal time for each trial, with a maximum of 50s. Three trials are conducted for each planner in each environment, resulting in a total of 4500 trials. The difficulty of the 300 environments are ordered from left to right based on the DWA performance (blue line). The performances of other planners are plotted as dots, for which a line is fit using linear regression. LfH 1.0 (red) fails many trials. The reason is that learning from a dataset with varying speed and always hallucinating the most constrained spaces causes ambiguity: the same most constrained C-space can map to different plans, which confuses the learner. LfH 0.4 (magenta) and HLSD 0.4 (cyan), both trained on the mostly constant 0.4m/s dataset without ambiguity, achieve similar performance and are less sensitive to navigational difficulty, especially LfH 0.4, but LfH 0.4 requires runtime hallucination and other components. Note that DWA has max speed of 1.0m/s, while LfH 0.4 and HLSD 0.4 have up to 0.6m/s max speed, modulated from mostly 0.4m/s learned from the dataset. DWA has an advantage in easy environments (left) due to its fast speed, but in difficult ones (right), LfH 0.4 and HLSD 0.4 are more stable. Also having 1.0m/s max speed, HLSD 1.0 considers the varying linear velocity in the 1.0m/s dataset with the \emph{offset}($\cdot$) function (line 9 in Alg. \ref{alg::sober}), and achieves the best result, outperforming all alternatives across the entire range of difficulties. 
The means and standard deviations of all five planners calculated from all trials are shown in Tab. \ref{tab::simulated_physical_results}. LfH 1.0 has the largest average time and variance, because the learner is confused by the varying training labels in the same most constrained spaces. DWA has large time and also high variance due to its sampling nature. Again, HLSD 0.4 performs similarly as LfH 0.4 does. 
HLSD 1.0 still outperforms all other planners.

\begin{table}
\centering
\caption{Simulated and Physical Traversal Time in Seconds}
\vspace{-5pt}
\begin{tabular}{cccccc}
\toprule
& \scriptsize{DWA} & \scriptsize{LfH 0.4} & \scriptsize{LfH 1.0} & \scriptsize{HLSD 0.4} & \textbf{\scriptsize HLSD 1.0} \\ 
\midrule
\scriptsize{Sim.} & \scriptsize{17.0$\pm$12.6} & \scriptsize{13.5$\pm$6.4}  & \scriptsize{26.7$\pm$15.0} & \scriptsize{13.4$\pm$9.8} & \textbf{\scriptsize 8.5$\pm$6.3}\\
\scriptsize{Phy.} & \scriptsize{78.8$\pm$10.0} & \scriptsize{67.0$\pm$4.4}  & \scriptsize{$\infty$} & \scriptsize{66.3$\pm$0.7} & \textbf{\scriptsize45.4$\pm$0.4}\\
\bottomrule
\end{tabular}
\vspace{-5pt}
\label{tab::simulated_physical_results}
\end{table}


\subsection{Physical Experiments}
We also deploy the same set of local planners in a physical test course, five trials each (Fig. \ref{fig::physical} top). The results are shown in Tab. \ref{tab::simulated_physical_results}. The complicated obstacle course causes DWA to execute many recovery behaviors, and the robot takes a long average time with large variance to finish the traversal. LfH 0.4 and HLSD 0.4 are both able to successfully navigate the robot through, with similar traversal times. But note that LfH 0.4 requires a fine-resolution global path for hallucination during deployment and the Turn in Place module, while HLSD does not and only reacts to the real obstacles without any other extra components. In this physical obstacle course, LfH 1.0 fails every trial due to multiple collisions. Our HLSD 1.0 achieves the best performance both in terms of average time and standard deviation. HLSD deployment in a natural cluttered environment is shown in Fig. \ref{fig::physical} bottom. 

\begin{figure}
  \centering
  \includegraphics[width=1\columnwidth]{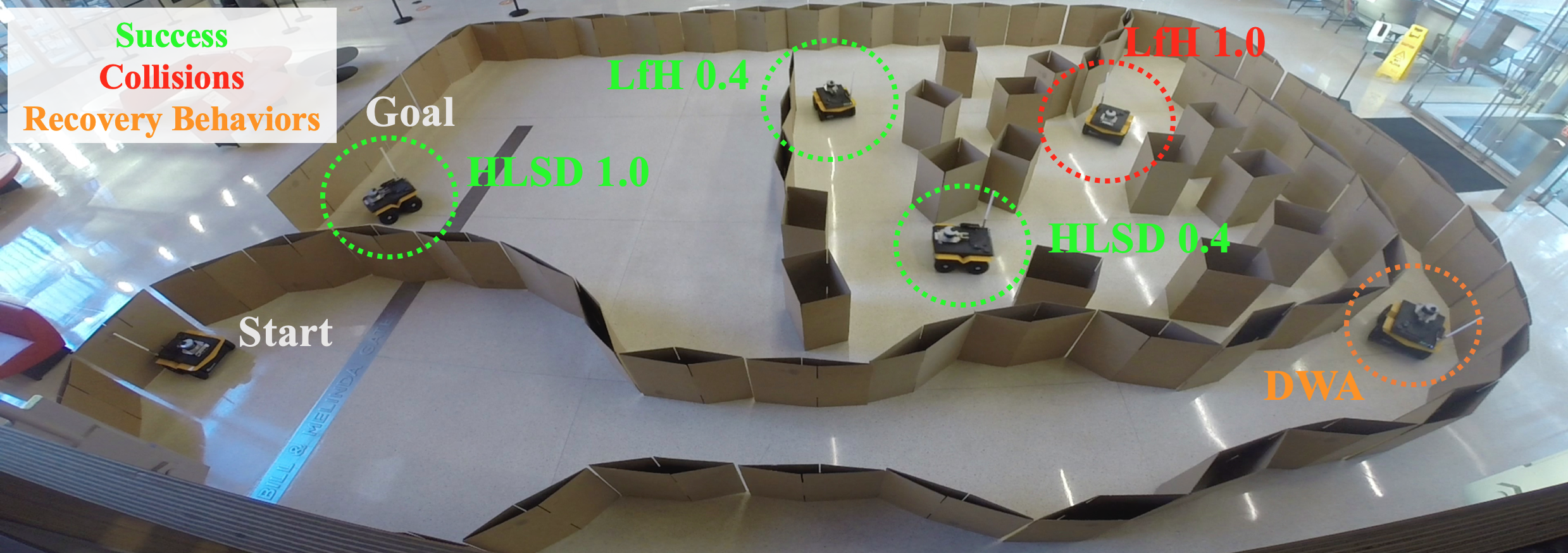}
  \includegraphics[width=1\columnwidth]{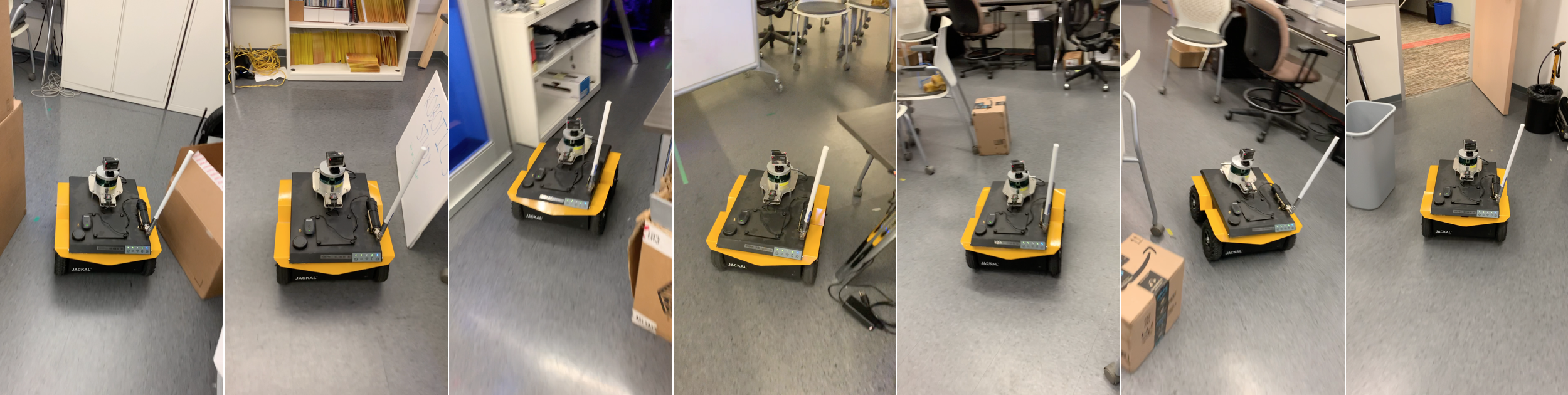}
  \vspace{-15pt}
  \caption{Physical Experiments (\url{https://www.youtube.com/watch?v=LZcBN9zgtXg&t=11s})}
  \label{fig::physical}
  \vspace{-10pt}
\end{figure}


\section{CONCLUSIONS}
\label{sec::conclusions}
We introduce HLSD, a self-supervised machine learning technique for mobile robot navigation with safety guarantee during training. Similar to LfH~\cite{xiao2021toward}, the robot safely learns in a completely obstacle-free environment and its perception is hallucinated with obstacle configurations where the actions taken in the free space remain optimal. Instead of overfitting to the most constrained spaces during training and requiring runtime hallucination and other modules to adapt to actual environments, the new HLSD pipeline allows the robot to navigate with the learned planner alone in response to realistic perception. By leveraging a large body of carefully designed hallucinations used for training, the learned planner does not need to deal with many out-of-distribution scenarios and has its own capability to address uncertainty/safety in the real-world during deployment. 
\bibliographystyle{IEEEtran}
\bibliography{IEEEabrv,references}

\end{document}